\documentclass{article}
\usepackage[margin=1in]{geometry}
\usepackage[parfill]{parskip}

\usepackage{amsmath, amssymb, amsfonts} 
\usepackage{nicefrac}
\usepackage{mathtools}
\usepackage{bm, bbm}
\usepackage[scr=boondoxo,scrscaled=1.05]{mathalfa}

\usepackage{enumitem}

\usepackage{xcolor}
\usepackage{comment}
\usepackage{soul}

\usepackage{hyperref}
\usepackage{cleveref}
\usepackage[numbers]{natbib}

\usepackage{tikz}
\usepackage{thm-restate}
\usepackage{appendix}

\usepackage{blindtext}

\usepackage{afterpage}

\usepackage{algorithm, algorithmic}    

\usepackage{times}

\definecolor{orange}{rgb}{1,0.4,0.0}

\DeclarePairedDelimiterXPP{\KL}[2]{D_{\textnormal{KL}}}{(}{)}{}{%
#1\:\delimsize\|\:#2%
}

\DeclarePairedDelimiterXPP{\RD}[2]{D_{$\alpha$}}{(}{)}{}{%
#1\:\delimsize\|\:#2%
}

\DeclarePairedDelimiterXPP\Prob[1]{\mathbb{P}}{\lbrace}{\rbrace}{}{

#1}

\DeclarePairedDelimiterXPP{\lnorm}[2]{}{\lVert}{\rVert}{_{#2}}{#1}

\newcommand{\bE}{\ensuremath{\mathbb{E}}}

\newcommand{\bP}{\ensuremath{\mathbb{P}}}
\newcommand{\bQ}{\ensuremath{\mathbb{Q}}}
\newcommand{\bR}{\ensuremath{\mathbb{R}}}

\newcommand{\cA}{\ensuremath{\mathcal{A}}}

\newcommand{\cH}{\ensuremath{\mathcal{H}}}

\newcommand{\cN}{\ensuremath{\mathcal{N}}}
\newcommand{\cO}{\ensuremath{\mathcal{O}}}

\newcommand{\mi}{\textup{I}}
\newcommand{\ent}{\textup{H}}

\newcommand{\REG}{\ensuremath{\textnormal{REG}}}

\newcommand{\innerproduct}[2]{\langle #1, #2 \rangle}

\newcommand{\kl}{\textup{D}_{\textnormal{KL}}}

\newtheorem{lemma}{Lemma}

\newtheorem{definition}{Definition}

\newtheorem{theorem}{Theorem}
\newtheorem{corollary}{Corollary}

\newtheorem{assumption}{Assumption}

\newtheorem{proof}{Proof}

\newcommand\blfootnote[1]{%
  \begingroup
  \renewcommand\thefootnote{}\footnote{#1}%
  \addtocounter{footnote}{-1}%
  \endgroup
}

\title{\textsc{Chained Information-Theoretic bounds and Tight Regret Rate for Linear Bandit Problems}}

\author{%
Amaury Gouverneur, Borja Rodríguez-Gálvez, Tobias J. Oechtering, and Mikael Skoglund \\
  KTH Royal Institute of Technology \\
  \texttt{\{amauryg, borjarg, oech, skoglund\}@kth.se} \\
}

\date{}

\begin{document}

\maketitle

\begin{abstract}%
  This paper studies the Bayesian regret of a variant of the Thompson-Sampling algorithm for bandit problems. It builds upon the information-theoretic framework of~\cite{russo_information-theoretic_2015} and, more specifically, on the rate-distortion analysis from~\cite{dong_information-theoretic_2020}, where they proved a bound with regret rate  of $O(d\sqrt{T \log(T)})$ for the $d$-dimensional linear bandit setting. We focus on bandit problems with a metric action space and, using a chaining argument, we establish new bounds that depend on the metric entropy of the action space for a variant of Thompson-Sampling. 
  Under suitable continuity assumption of the rewards, our bound offers a tight rate of $O(d\sqrt{T})$ for $d$-dimensional linear bandit problems.  
\blfootnote{This work was partially supported by (i) the Wallenberg AI, Autonomous Systems and Software Program (WASP) funded by the Knut and Alice Wallenberg Foundation and (ii) the Swedish Research Council under contract 2019-03606.}
\end{abstract}

\section{Introduction}
Bandit problems are a class of decision problems in which an agent interacts sequentially with an unknown environment by choosing actions and earning rewards in return. The goal of the agent is to maximize its expected cumulative reward, which is the expected sum of rewards that it will earn throughout its interaction with the environment. This necessitates a delicate balance between the exploration of different actions to gather information for potential future rewards, and the exploitation of known actions to receive immediate gains. The theoretical study of the performance of an algorithm in a bandit problem is done by analyzing the \emph{expected regret}, which is defined as the difference between the cumulative reward of the algorithm and the hypothetical cumulative reward that an oracle would obtain by choosing the optimal action at each time step. An effective method for achieving small regret is the Thomson Sampling (TS) algorithm \citep{thompson1933likelihood}, which, despite its simplicity, has shown remarkable performance \citep{russo2018tutorial,russo_learning_2017,chapelle2011empirical}.\\
 
Studying the Thomspon Sampling regret, \cite{russo_information-theoretic_2015}  introduced the concept of information ratio, a statistic that captures the trade-off between the information gained by the algorithm about the environment and the immediate regret. They used this concept to provide a general upper bound for finite action spaces $\cA$ that depends on the entropy of the optimal action $\ent(A^\star)$, the time horizon $T$ (the total number of times that the agent interacts with the environment), and a problem-dependent upper bound on the information-ratio $\Gamma$, namely $\sqrt{\Gamma\cdot T\cdot \ent(A^\star)}$. For finite environment parameter spaces, under a Lipschitz continuity assumption of the expected reward and using Lipschitz maximal inequality argument, \citet{dong_information-theoretic_2020} were able to control the regret of the TS algorithm via a "compressed statistic" $\Theta_\varepsilon$ of the environment paramers $\Theta$, with a bound of the form $\varepsilon\cdot T + \sqrt{\Gamma\cdot T\cdot \ent(\Theta_\varepsilon)}$. In particular, they derived a near optimal regret rate of $O(d \sqrt{T \log T})$ for $d$-dimensional linear bandit problems.\\

In this paper, building on the work of \cite{dong_information-theoretic_2020}, we explored the use of the chaining technique for bandit problems where the rewards exhibit some subgaussian continuity property with respect to the action space. We introduced the \emph{Two Steps Thompson Sampling} (2-TS), a variant of the original algorithm where the history is updated every time steps. For this algorithm, we derive a bound that captures the continuity property of the reward process and depends on the metric entropy of the action space. Notably our bound does not require finite environment or action space and holds for continuous action spaces. For the class of linear bandit problems, we obtained a bound in $O(d\sqrt{T})$ matching the best possible regret $\Omega( d \sqrt{T} )$ from \cite{dani_stochastic_2008}.\\

The rest of the paper is organized as follows. Section~\ref{sec:problem_setup} presents the bandit problem setup, gives a definition of the Bayesian expected regret, introduces Two Steps Thompson Sampling algorithm and the specific notations. Section~\ref{sec:chaining_technique}
explains the idea of the bounding technique, defines the required tools and assumptions we will be using. Section~\ref{sec:main_results} states and proves our main theorem. Section~\ref{sec:application} applies our theorem to the important case of linear bandit problems and derives several specific bounds before giving an deriving a bound for linear bandit problems with a ball structured action space. Finally, Section~\ref{sec:conclusion} discusses our results and possible extensions and future work.

\section{Problem setup}
\label{sec:problem_setup}
We consider a sequential decision problem, where at each time step (or round) $t\in \{1,\ldots,T\}$, an agent interacts with an environment by selecting an action $A_t$ from a action set $\cA$ and, based on that action, receives a real valued reward $R_t\in \mathbb{R}$. The pair of the selected action and the received reward is collected in a history $H^{t+1} = H^t \cup H_{t+1}$, where $H_{t+1}=\{A_t,R_t\}$, that will be accessible to the agent in the next round. The procedure repeats until the last round $t=T$.  \\

Following the Bayesian framework, we consider the environment to be characterized by some parameters $\theta \in \cO$, unknown to the agent, that are sampled from a known prior distribution $\bP_{\Theta}$. This prior, together with the reward distribution $\bP_{R|A,\Theta}$,fully describes the bandit problem. As the reward distribution depends on the selected action and the environment parameters, it may be written as $R_t = R(A_t,\Theta)$ for some possibly random function $R:\cA\times\cO\to\bR$.\\

The goal of the agent is to take a sequence of actions that maximizes the total collected reward. More specifically, the agent seeks to learn a policy $\varphi=\{\varphi_t:\cH^t\to\cA\}_{t=1}^T$ that, for each time $t \in \{1,\ldots,T \}$, selects an action $A_t$ based on the history $H^t$ such that it maximizes the \emph{expected cumulative reward} $R_T(\varphi) \coloneqq \bE \left[\sum_{t=1}^T R(\varphi_t(H^t),\Theta)\right]$.\\

\subsection{The Bayesian expected regret}

The Bayesian expected regret quantifies the difference between the expected cumulative reward achieved by the agent following a policy $\varphi$ and the optimal expected cumulative reward that could be obtained by an \emph{omniscient} agent having access to the true reward function and selecting the action yielding the highest expected reward. 
\begin{definition}[Optimal cumulative reward]
    \label{def:optimal_cumulative_reward}
    The \emph{optimal cumulative reward} of a bandit problem is defined as
    \begin{equation*}
        R^\star_T \coloneqq \sup_{\psi} \bE \bigg[ \sum_{t=1}^T R( \psi(\Theta), \Theta) \bigg],
    \end{equation*}
    where the supremum is taken over all decision rules $\psi: \cO \to \cA$ such that the expectation above is defined.
\end{definition}

We denote a policy that achieves the supremum of~\Cref{def:optimal_cumulative_reward} as $\psi^\star$ and we refer to the action it selects as the \emph{optimal action}  
$A^\star \coloneqq \psi^\star(\Theta)$. We make the following technical assumption on the action set to ensure that such a policy exists. 

\begin{assumption}[Compact action set]
    \label{ass:compact_action_set}
    The set of actions $\cA$ is compact. 
\end{assumption}

The difference between the optimal cumulative reward and the expected cumulative reward of a policy $\varphi$ is called the Bayesian expected regret of $\varphi$, denoted $\REG_T(\varphi)$.

\begin{definition}[Bayesian expected regret]
\label{def:bayesian_expected_regret}
    The \emph{Bayesian expected regret} of a policy $\varphi$ in a bandit problem is defined as
    \begin{equation*}
        \REG_T(\varphi) \coloneqq R^\star_T - R_T(\varphi).
    \end{equation*}
\end{definition}

\subsection{Thompson Sampling algorithm and the Two Steps variant}
\label{subsec:thompson_sampling}

One of the most popular and most studied algorithm for solving bandit problems is the \emph{Thompson Sampling} (TS) algorithm \textcolor{blue}{\cite{russo2018tutorial,russo_learning_2017,chapelle2011empirical,dong_information-theoretic_2020}}. TS works by sampling a Bayesian estimate of the environment parameters from the posterior distribution, and taking the optimal action for the sampled estimate. Specifically, at each time step $t \in \{ 1,\ldots,T \}$, the agent draws a Bayesian estimate $\hat{\Theta}_t$ based on the past collected history $H^t$, takes the corresponding optimal action $\hat{A_t} = \psi^\star(\hat{\Theta}_t)$, receives a reward $R_t$, and updates the history $H^{t+1} = \{H^t,\hat{A}_t,R_t\}$.

In this work, we consider a variation of TS, that we refer to as \emph{Two Steps Thompson Sampling} (2-TS). The key difference between this algorithm and the TS algorithm is that the history is updated every two time steps\footnote{We implicitly assume that, for Two Steps Thompson Sampling, the total number of steps $T$ is an even number.}. Intuitively, the algorithm will behave the same, but will wait to have collected two rewards before updating its history, thus taking slightly less informed actions half of the time. This difference will be important to implement the chaining technique later on. The pseudocode for Two Steps Thompson Sampling is given in~\Cref{alg:Two_Steps_Thompson_Sampling}. 
\begin{algorithm}[ht]
    \caption{Two Steps Thompson Sampling algorithm}
    \label{alg:Two_Steps_Thompson_Sampling}
    \begin{algorithmic}[1]
        \STATE {\bfseries Input:} environment parameters prior $\bP_{\Theta}$.
        \FOR{$t=1$ {\bfseries to} T}
            \STATE Sample a parameter estimation $\smash{\hat{\Theta}_t \sim \bP_{\Theta|H^t}}$.
            \STATE Take the corresponding optimal action $\hat{A}_t = \psi^\star(\hat{\Theta}_t)$.
            \STATE Collect the reward $R_t = R(\hat{A}_t, \Theta)$.
            \IF{$t$ is even}
            \STATE Update the history $H^{t+1}=\{H^t,\hat{A}_t,R_t,\hat{A}_{t-1},R_{t-1}\}$.
            \ELSE
            \STATE Keep the history $H^{t+1}=H^t$.
            \ENDIF
        \ENDFOR
    \end{algorithmic}
\end{algorithm}

\subsection{Notation specific to bandit problems}
\label{subsec:notations}

Since the $\sigma$-algebras of the history $H^t$ are often used in the conditioning of the expectations and probabilities coming up in the analysis, similarly to~\cite{russo_information-theoretic_2015,dong_information-theoretic_2020, neu_lifting_2022, gouverneur_thompson_2023}, we define the operators $\bE_t[\cdot] \coloneqq \bE[\cdot|H^t]$ and $\bP_t[\cdot] \coloneqq \bP[\cdot|H^t]$, whose outcomes are $\sigma(\cH^{t})$-measurable random variables and $\cH = \cA \times \bR$.

Analogously, we define $\mi_t(A^\star; R_t) \coloneqq \bE_t[ \kl( \bP_{R_t | H^t, A^\star} \lVert \bP_{R_t | H^t} )]$ as the \emph{disintegrated} conditional mutual information between the optimal action $A^\star$ and the reward $R_t$, \emph{given the history $H^t$}, see~\cite[Definition 1.1]{negrea_information-theoretic_2020}, which is itself also a $\sigma(\cH^t)$-measurable random variable.

When it is clear from the context that the random rewards depend on the environment parameters $\Theta$, we will often use the notation $R(A_t)$ as a shorthand for $R(A_t,\Theta)$ to simplify the expressions.

\section{Chain-link Information Ratio and Chaining Technique}
\label{sec:chaining_technique}

In bandit problems where the rewards of nearby actions exhibit some continuity property, we aim to exploit this dependence using a chaining argument. More specifically, our idea is to approach the \emph{Two Step Thompson Sampling} algorithm by a chain of increasingly accurate approximations, which we refer to as \emph{``approximate learning"}. \\ 

Inspired by~\cite{dong_information-theoretic_2020}, our construction relies on the existence of a sequence of finer and finer quantizations $\{A^\star_k\}_{k=k_0}^\infty$ of the optimal action $A^\star$ and a corresponding carefully crafted action sampling function $f_t^k: \cA \to \cA$ for each round $t \in \{ 1, \ldots, T \}$. This quantization and sampling functions are designed to satisfy the following three requirements simultaneously:

\begin{enumerate}[label=(\roman*)]
    \item \label{requirement:less_informative} The quantizations $A^\star_k$ are less informative than $A^\star$, that is, $\ent(A^\star_k) \leq \ent(A^\star)$ for all $k \geq k_0$.
    \item \label{requirement:chaining_sum}At each round $t\in \{1,\ldots,T\}$, the \emph{Two Step Thompson Sampling} regret can be written as an infinite sum of the difference between the approximate learning regrets:
    \begin{align*}
        \bE_t &\Big[R(A^\star) - R(\hat{A}_t)\Big] = \\ &\sum_{k=k_0+1}^{\infty} \bE_t \Big[\Big(R(f_t^k(A^\star_k)) - R(f_t^k(\hat{A}_{t,k}))  \Big)
        - \Big(R(f_t^{k-1}(A^\star_{k-1}))-R(f_t^{k-1}(\hat{A}_{t,k-1})) \Big)\Big].
    \end{align*}
    \item \label{requirement:no_more_information}For each time step $t\in\{1,\ldots,T\}$, and for every $k > k_0$, the regret difference between the $k^{\textnormal{th}}$-consecutive \emph{``approximate learning"} can be bounded using the information gained about the quantization $A^\star_k$ while, at the same time, it reveals no more information about the quantization $A^\star_k$ than Two Step Thompson Sampling. 
\end{enumerate}

\subsection{Nets and quantizations}

When designing the quantization $A^\star_k\in \cA_k$ of the optimal action, we face two conflicting goals: on the one hand, we want the quantization to be little informative about $A^\star$ while, on the other hand, we want to ensure that $\cA_k$ converges quickly to a good approximation of $\cA$. This dual objective naturally leads to considering $\varepsilon$-nets.

\begin{definition}[$\varepsilon$-net and covering number]
A set $\cN$ is called an \emph{$\varepsilon$-net} for $(\cA,\rho)$ if, for every $a \in \cA$, there is a $\pi(a) \in \cN$ such that $\rho(a,\pi(a)) \leq \varepsilon$. The smallest cardinality of an $\varepsilon$-net for $(\cA,\rho)$ is called the \emph{covering number}, that is
\begin{equation*}
    \cN(\cA,\rho,\varepsilon) \triangleq \inf \big \lbrace |\cN| : \textnormal{ $\cN$ is an $\varepsilon$-net of $(\cA,\rho)$} \big \rbrace.
\end{equation*}
\end{definition}

The covering number $\cN(\cA,\rho,\varepsilon)$ can be understood as a measure of the complexity of the action set $\cA$ at the resolution $\varepsilon$.
Equipped with this new concept, a possible $k^{\textnormal{th}}$-quantization $A^\star_k$ is the quantization of the optimal action $A^\star$ at the scale $2^{-k}$.

\begin{definition}[$k^{\textnormal{th}}$-quantization]
\label{def:kth_quantization}
    Let $\cA_k$ be a $2^{-k}$-net for $(\cA,\rho)$ with an associated mapping $\pi_k : \cA \to \cA_k$, such that the mappings $\pi_k$ are restricted to those of the form $\pi_k = \pi_k' \circ \pi_{k+1}$, where $\pi_k': \cA_{k+1} \to \cA_k$. We define $A^\star_k = \pi_k(A^\star)$ as the $k^{\textnormal{th}}$-quantization of the optimal action $A^\star$ with respect to $(\cA,\rho)$. Similarly, the quantization $\hat{A}_{t,k} = \pi_k(\hat{A}_t)$ is the $k^{\textnormal{th}}$-quantization of the sampled action $\hat{A}_t$.
\end{definition}
Note that $A^\star_{k}$ is completely determined by $A^\star_{k+1}$ via the mapping $\pi_{k}':\cA_{k+1} \to \cA_{k}$.
In the following, we set $k_0$ to be the largest integer such that $2^{-k_0} \geq \textnormal{diam}(\cA)$. 
\subsection{Existence of the \emph{``approximate learning"}}

The sequence of quantizations $\{A^\star_k\}_{k=k_0}^\infty$ given in Definition~\ref{def:kth_quantization} satisfy Requirement~\ref{requirement:less_informative} since there is a deterministic mapping between $A^\star$ and $A^\star_k$~\citep[Theorem 1.4 (f)]{yury_polyanskiy_information_2022}.  We claim that for each time step $t \in \{1,\ldots,T\}$, and for each $k > k_0$, there exists a random function $f_t^k: \cA_k \to \cA_k$ that satisfies Requirements \ref{requirement:chaining_sum} and \ref{requirement:no_more_information}.

\begin{restatable}{proposition}{ExistenceApproximateLearning}
\label{prop:existence_learning}
Let $\{A^\star_k\}_{k=k_0}^{\infty}$ be defined as in Definition \ref{def:kth_quantization}. For each time step $t \in \{1,\ldots,T\}$, there exists a sequence of random functions $\{f_t^k\}_{k=k_0}^{\infty}$ that for each $k > k_0$, satisfies the following:
\begin{enumerate}[label=(\roman*)]
\item \label{prop:subprop_singleton}$\bE_t \Big[R(f_t^{k_0}(A^\star_{k_0})) - R(f_t^{k_0}(\hat{A}_{t,k_0}))\Big] = 0,$
\item \label{prop:subprop_limit}$\lim_{k\to \infty} \bE_t \Big[R(f_t^k(A^\star_k)) - R(f_t^k(\hat{A}_{t,k}))\Big] = \bE_t \Big[R(A^\star) - R(\hat{A}_t)\Big],$ \textnormal{ and}
\item \label{prop:subprop_no_more_information}$\mi_{t} \big(A^\star_k;  R(f_t^k(\hat{A}_{t,k})),R(f_t^{k-1}(\hat{A}_{t,k-1})) \leq \mi_{t} (A^\star_k; R(\hat{A}_t),R(\hat{A}_{t}')\big)$, a.s. 
\end{enumerate}
where in \ref{prop:subprop_no_more_information} the sampled actions $\hat{A}_t$ and $\hat{A}_t'$ are identically distributed.
\end{restatable}

\begin{proof}
    The proof follows closely the proof of~\citep[Proposition 2]{dong_information-theoretic_2020} and is given in Appendix~\ref{app:existence_learning}.
\end{proof}

\subsection{Subgaussian process, smooth rewards and chain-link information ratio}
 
The motivation for using a chaining technique is our aim to derive a regret bound that could capture effectively the dependence between the rewards of nearby actions.
We conceptualize this dependence considering that the rewards are subgaussian with respect to the actions.

\begin{definition}[Subgaussian process]
\label{def:subgaussian_process}
A stochastic process $\lbrace R_a \rbrace_{a \in \cA}$ on the metric space $(\cA,\rho)$ is called \emph{subgaussian} if 
for all $a, b \in \cA$ and all $\lambda > 0$
\begin{equation*}
    \log \bE \bigg[ e^{\lambda(R_a - R_b)} \bigg] \leq \frac{\lambda^2 \rho(a,b)^2}{2}.
\end{equation*}
\end{definition}
Technically, for a process $\{ R_a \}_{a \in \ cA}$ to be subgaussian it is also required that $\bE[R_a] = 0$ for all $a \in \cA$, see, for example~\citep[Definition 5.20]{van_handel_probability_2016}. However, we do not require this restriction moving forward. One way to interpret the subgaussian process property is to understand it as an "in probability continuity" requirement. Actually, Definition \ref{def:subgaussian_process}, up to constant terms, can be equivalently written as 
\begin{align*}
    \bP[|R_a-R_b|\geq t]\leq 2 e^{-Ct^2 \rho(a,b)^2}
\end{align*}
for all $t\geq 0$ and all all $a, b \in \cA$, and for some $C > 0$.\\ 

Lastly, to ensure that the difference of regret between consecutive \emph{approximate learning} vanishes asymptotically, we can impose the following mild technical assumption.

\begin{definition}[Separable process]
\label{def:seperable_process}
A stochastic process $\lbrace R_a \rbrace_{a \in \cA}$ is called \emph{separable} if there is a countable set $\cA' \subseteq \cA$ such that, for all $a \in \cA$

\begin{equation*}
    R_a \in \lim_{\substack{a'\to a \\ a' \in \cA'}} R_{a'} \ \textnormal{ a.s.}
\end{equation*}

\end{definition}

We refer to rewards satisfying both definition \ref{def:subgaussian_process} and \ref{def:seperable_process} as \emph{smooth rewards} on the metric space $(\cA,\rho)$.

\begin{definition}[Smooth rewards]
    We say that the rewards are \emph{smooth on the metric space $(\cA,\rho)$}, if for all environment parameters $\theta \in \cO$, the random rewards $\{R(a,\theta)\}_{a\in \cA}$ form a separable subgaussian process on $(\cA,\rho)$.
\end{definition}

To control the difference of regret between successive \emph{approximate learning}, it is useful to introduce the concept of \emph{chain-link information ratio}. It is a direct adaptation of our chaining technique to the \emph{information ratio} introduced by~\cite{russo_information-theoretic_2015} and later used by~\cite{dong_information-theoretic_2020}. 

\begin{definition}[Chain-link information ratio]
For each time step $t\in \{1,\ldots,T\}$, and for each $k > k_0$, we define the \emph{chain-link information ratio} as
\begin{equation*}
\Gamma_{t,k} \coloneqq \frac{\bE_t \Big[\Big(R(f_t^k(A^\star_k))- R(f_t^k(\hat{A}_{t,k}))  \Big)-\Big(R(f_t^{k-1}(A^\star_{k-1}))-R(f_t^{k-1}(\hat{A}_{t,k-1})) \Big)\Big]^2}{\mi_t(f_t^k(A^\star_k),f_t^{k-1}(A^\star_{k-1});R(f_t^k(\hat{A}_{t,k})),R(f_t^{k-1}(\hat{A}_{t,k-1})))}
\end{equation*}
where $A^\star_k,A^\star_{k-1}$ and $\hat{A}_{t,k},\hat{A}_{t,k}$ are the $k^{\textnormal{th}}$ and $(k-1)^{\textnormal{th}}$ quantizations of $A^\star$ and $\hat{A}_t$ respectively and where the random functions $f_{t}^k$ and $f_t^{k-1}$ satisfy the conditions of Proposition \ref{prop:existence_learning},
\end{definition}
There is no particular interpretation of the chain-link information ratio. The purpose of its introduction is to unify elegantly specific results via problem-dependent upper bounds on $\Gamma_{t,k}$ similarly to what is done in prior works for the information ratio~\citep{russo_information-theoretic_2015, dong_information-theoretic_2020} and the lifted information ratio~\citep{neu_lifting_2022, gouverneur_thompson_2023}. 

\section{Main result}
\label{sec:main_results}

In this section, we leverage the previously introduced concepts to derive a general chained bound on the Two Steps Thompson Sampling regret for bandit problems with smooth rewards.  We obtain a bound that depends on the complexity of the action space. Remarkably, through the use of Lemma \ref{lemma:new_lemma}, our result hold for continuous action spaces. We note that \ref{lemma:new_lemma} could be applied to~\cite{dong_information-theoretic_2020} as a generalization of their~\cite[Lemma 1]{dong_information-theoretic_2020}, thus extending their results to infinite and continuous environment spaces.

\begin{theorem}[Chained bound]
\label{thm:main_theorem}
For bandit problems with smooth rewards on the metric space $(\cA,\rho)$, the \emph{2-TS} expected cumulative regret after $T$ steps is bounded as
\begin{align}
\REG^{\textnormal{2-TS}}_T \leq \sum_{k=k_0+1}^\infty \sqrt{ 2\cdot \bar{\Gamma}_k \cdot T  \cdot \ent(A^\star_k) },
       \nonumber
\end{align} 
where $A^\star_k$ is the $k^\textnormal{th}$-quantization about the optimal action $A^\star$ with respect to $(\cA,\rho)$ and where for each $k > k_0$, and $\bar{\Gamma}_k$ is a upper bound on $\bE[\Gamma_{t,k}]$. 
\end{theorem}

\begin{proof}
We show that
    \begin{align*}
        \REG^{\textnormal{2-TS}}_T &= \sum_{t=1}^T \bE[R(A^\star)-R(\hat{A}_t)]\\
        &\stackrel{(a)}{=} 2 \sum_{1\leq t\leq T, t \textnormal{ odd}} \bE[R(A^\star)-R(\hat{A}_t)]\\
        &= 2 \sum_{1\leq t\leq T, t \textnormal{ odd}} \bE\Big[\bE_t[R(A^\star)-R(\hat{A}_t)]\Big]\\
        &\stackrel{(b)}{=} 2 \sum_{1\leq t\leq T, t \textnormal{ odd}} \bE\Big[ \sum_{k=k_0+1}^{\infty} \bE_t \Big[\Big(R(f_t^k(A^\star_k))- R(f_t^k(\hat{A}_{t,k}))  \Big)\\
&\qquad\qquad\qquad\qquad\qquad\qquad-\Big(R(f_t^{k-1}(A^\star_{k-1}))-R(f_t^{k-1}(\hat{A}_{t,k-1})) \Big)\Big]\Big]\\
&\stackrel{(c)}{\leq} 2 \sum_{1\leq t\leq T, t \textnormal{ odd}} \sum_{k=k_0+1}^{\infty}\bE\Big[\sqrt{\Gamma_{t,k} \cdot \mi_t(A^\star_k,A^\star_{k-1};R(f_t^k(\hat{A}_{t,k})),R(f_t^{k-1}(\hat{A}_{t,k-1})))}\Big]\\
&\stackrel{(d)}{\leq} 2 \sum_{1\leq t\leq T, t \textnormal{ odd}} \sum_{k=k_0+1}^{\infty}\sqrt{\bE[\Gamma_{t,k}] \cdot \mi(A^\star_k;R(\hat{A}_t),R(\hat{A}_{t+1})|H^t)}\\
&\stackrel{(e)}{\leq} 2\sum_{k=k_0+1}^{\infty} \sqrt{ \frac{T}{2} \cdot \bar{\Gamma}_k \cdot \sum_{1\leq t\leq T, t \textnormal{ odd}}\mi(A^\star_k;R(\hat{A}_t),R(\hat{A}_{t+1}))|H^t)}\\
&\stackrel{(f)}{\leq} \sum_{k=k_0+1}^{\infty}  \sqrt{2\cdot \bar{\Gamma}_k\cdot T \cdot \sum_{1\leq t\leq T, t \textnormal{ odd}}\mi(A^\star_k;\hat{A}_t,R(\hat{A}_t),\hat{A}_{t+1},R(\hat{A}_{t+1}))|H^t)}\\
&\stackrel{(g)}{=} \sum_{k=k_0+1}^{\infty}  \sqrt{ 2 \cdot \bar{\Gamma}_k  \cdot T \cdot\mi(A^\star_k;H^T)}\\
&\stackrel{(h)}{\leq} \sum_{k=k_0+1}^{\infty} \sqrt{ 2 \cdot \bar{\Gamma}_k  \cdot T \cdot\ent(A^\star_k)}
    \end{align*}
    where (a) follows since the history of the of 2-TS is being updated every 2 time steps; (b) follows from the definition of the approximate learning; (c) follows from the definition of $\Gamma_{t,k}$ and the data-processing inequality; (d) follows from consecutively using the fact that $A^\star_{k-1}$ is completely determined by  $A^\star_{k}$, then using Proposition \ref{prop:existence_learning} \ref{prop:subprop_no_more_information}, and finally applying Jensen's inequality ; (e) follows from the definition of $\bar{\Gamma}_k$ and the application of the Cauchy-Schwartz inequality; (f) results  from the "more data, more information" property~\citep[Proposition~2.3.5]{yury_polyanskiy_information_2022}; (g) follows from the chain rule for mutual information; and (h) comes from~\citep[Proposition~2.4.4]{yury_polyanskiy_information_2022} and the fact that $\cA_k$ is a finite set.
\end{proof}

In the next section, we present application of \Cref{thm:main_theorem} to derive explicit regret bounds for particular settings of bandit problems with structure and show that our bound offers a tight regret rate for the linear bandit problem.

\section{Applications to linear bandit problems}
\label{sec:application}

In \emph{linear bandits} problems, each action is parameterized by a feature vector and the associated expected reward can be written as the inner product between the feature vector and the environment parameter. Mathematically, a $d$-dimensional linear bandit problem is a bandit problem with $\cA,\cO \subset \bR^d$ and such that for all $a\in \cA$ and all $\theta \in \cO$ we have \begin{align*}
    \bE[R(a,\theta)] = \innerproduct{a}{\theta},
\end{align*}
where the expectation is taken over the randomness of the reward function.\\

Using a similar analysis as~\citet{russo_information-theoretic_2015}, we can bound the chain-link information ratio in linear bandits via the dimension of the action space. The proof is given in Appendix~\ref{app:chain-link_information_ratio_bound_for_smooth_d_linear_bandits}.

\begin{restatable}{proposition}{ChainLink}
\label{prop:chain-link_information_ratio_bound_for_smooth_d_linear_bandits}
For $d$-dimensional linear bandit problems with smooth rewards on the metric space $(\cA,\rho)$, for each $t\in\{1,\ldots,T\}$, and each $k>k_0$, we have that
\begin{align*}
\Gamma_{t,k} \leq 2\cdot(6\cdot 2^{-k})^2\cdot d,
\end{align*} 
where $\Gamma_{t,k}$ is the $k^\textnormal{th}$-chain-link information ratio.
\end{restatable}

Combining Proposition \ref{prop:chain-link_information_ratio_bound_for_smooth_d_linear_bandits} and \Cref{thm:main_theorem} leads to the following bound on the \emph{2-TS} regret for linear bandit problems with smooth rewards.

\begin{theorem}[Smooth linear bandit]
\label{thm:smooth_linear_bandit}
For $d$-dimensional linear bandit problems with smooth rewards on the metric space $(\cA,\rho)$, the \emph{2-TS} expected cumulative regret after $T$ steps is bounded by
\begin{align}
\REG^{\textnormal{2-TS}}_T &\leq 12 \sum_{k=k_0+1}^\infty 2^{-k}\sqrt{ d \cdot T  \cdot \ent(A^\star_k) },
       \nonumber
\end{align} 
where $A^\star_k$ is the $k^\textnormal{th}$ quantization of the optimal action $A^\star$ with respect to the metric space $(\cA,\rho)$, as defined in Definition \ref{def:kth_quantization}. 
\end{theorem}

From \Cref{thm:smooth_linear_bandit}, we can derive a bound that depends on the \emph{entropy integral}. The proof follows the steps from~\citep[Corollary 5.25]{van_handel_probability_2016} and is given in Appendix~\ref{app:entropy_integral}.

\begin{corollary}[Entropy integral]
\label{cor:entropy_integral}
    For a linear bandit of dimension $d$, with smooth rewards on the metric space $(\cA,\rho)$, the \emph{2-TS} expected cumulative regret after $T$ steps is bounded as
\begin{align}
\REG^{\textnormal{2-TS}}_T
       &\leq 24 \sqrt{d \cdot T} \int_{0}^\infty \sqrt{ \log(|\cN(\cA,\rho,\varepsilon)|) } d\varepsilon,
       \nonumber
\end{align} 
where $\cN(\cA,\rho,\varepsilon)$ is the $\varepsilon$-net of smallest cardinality for $(\cA,\rho)$.
\end{corollary}

For linear bandit problems where the possible actions lie in the unit ball, with the help of a covering argument, we come to the following result. The proof is given in Appendix~\ref{app:linear_smooth_bandit_unit_ball}.

\begin{restatable}{proposition}{LinearSmoothBanditUnitBall}
\label{prop:linear_smooth_bandit_unit_ball}
For $d$-dimensional linear bandits with smooth rewards with respect to $(\cA,||.||_2)$ and a ball-structured action space $\cA\subseteq \overline{\mathbf{B}_d(0,1)}$, where $\overline{\mathbf{B}_d(0,1)}$ is the $d$-dimensional closed Euclidean unit ball, the \emph{2-TS} expected cumulative regret is bounded as
\begin{align}
\REG^{\textnormal{2-TS}}
       &\leq 7\cdot d  \sqrt{T}.
       \nonumber
\end{align} 
\end{restatable}

The remarkable property of the above bound is that it is the first information-theoretic bound on the regret of an algorithm for linear bandits problem that only depends on the dimension $d$ and the square root of the total number of steps $T$. It improves on the bound $O(d\sqrt{T \log(T)})$ from~\citet[Theorem 2]{dong_information-theoretic_2020} and matches the the minimax lower bound $\Omega(d\sqrt{T})$ proven by~\citet[Theorem 3]{dani_stochastic_2008} thus suggesting that Two Steps Thompson Sampling is optimal in this context.

\section{Conclusion}
\label{sec:conclusion}
In this paper, we studied bandit problems with rewards that exhibit some continuity property with respect to the action space. We have introduced a variation of the Thompson Sampling algorithm, which we named the Two Step Thompson Sampling. The sole difference between this algorithm and the original Thompson Sampling algorithm is that the history is updated every two time steps. In \Cref{thm:main_theorem}, we have demonstrated using a chaining argument that the Two Steps Thompson Sampling cumulative expected regret is bounded from above by a measure of the complexity of the action space. For $d$-dimensional linear bandit problems where the rewards form a subgaussian process with respect to the action space, we obtain a tight regret rate $O(d\sqrt{T})$ that improves upon the best information-theoretic bounds and matches with the minimax lower bound $\Omega(d \sqrt{T})$~\citep{dani_stochastic_2008}. Our results raise the question whether it is possible to obtain such bounds for the original Thompson Sampling algorithm regret, either via adapting our proof techniques or by relating it to the Two Steps Thompson Sampling regret. One could imagine analyzing separately the regrets of the odd time steps than those of the even time steps and try to apply techniques as in our paper.  
Future work also include extending our results to generalized linear bandits and logistic bandit problems.

\bibliography{references.bib,references_isit.bib}
\bibliographystyle{IEEEtranN}

\appendix

\crefalias{section}{appendix} 
\newpage
\section{Additional Lemmata}

\begin{lemma}
\label{lemma:new_lemma}
    Consider a space $\cA$, two functions $f: \cA \to \bR_+$ and $g: \cA \to \bR_+$, and a probability distribution $\bQ$ on $\cA$. Then, there exists a pair $(a_1, a_2) \in \cA^2$ and a $q \in [0,1]$ such that
    \begin{equation*}
        \label{eq:lemma_smaller}
        q f(a_1) + (1-q) f(a_2) \leq \int_{a\in \cA} f(a) \mathrm{d} \bQ(a) \quad \textnormal{and} \quad q g(a_1) + (1-q) g(a_2) \leq \int_{a\in \cA} g(a) \mathrm{d} \bQ(a).
    \end{equation*}
\end{lemma}

\begin{proof}
    The proof is inspired by to the one from~\citet[Lemma 2]{dong_information-theoretic_2020}. Although, it contains key modifications that allow this version of the lemma to work for general spaces $\cA$ that are not necessarily finite.

    Let $\bar{F} = \int_{a \in \cA}  f(a) \mathrm{d} \bQ(a)$ and $\bar{G} = \int_{a \in \cA} g(a) \mathrm{d} \bQ(a)$. Now, consider the spaces $\cA_f \coloneqq \{ a \in \cA: f(a) \leq \bar{F} \}$ and $\cA_g \coloneqq \{ a \in \cA : g(a) \leq \bar{G} \}$. If $\cA_f \cap \cA_g \neq \emptyset$, then taking both $a_1$ and $a_2$ from $\cA_f \cap \cA_g$  trivially satisfies the conditions for all $q \in [0,1]$. Therefore, let us assume that the sets are disjoint for the rest of the proof. 

    Consider some $a_1 \in \cA_f = \cA_g^c$ and some $a_2 \in \cA_g = \cA_f^c$. The required condition from the lemma can be re-written as 
    \begin{equation*}
        q \geq \frac{f(a_2) - \bar{F}}{f(a_2) - f(a_1)} \quad \textnormal{ and } \quad q \leq \frac{\bar{G} - g(a_2)}{g(a_1) - g(a_2)},
    \end{equation*}
    where the first inequality took into account that $f(a_1) < f(a_2)$ by the definition of the sets $\cA_f$ and $\cA_g = \cA_f^c$. This inequality can, in turn, be written as
    \begin{equation*}
        \frac{f(a_2) - \bar{F}}{f(a_2) - f(a_1)} \leq \frac{\bar{G} - g(a_2)}{g(a_1) - g(a_2)} 
    \end{equation*}
    which is equivalent to
    \begin{equation*}
        f(a_2) g(a_1) - \bar{F} \big( g(a_1) - g(a_2) \big) \leq \bar{G} \big( f(a_2) - f(a_1) \big) + f(a_1) g(a_2).
    \end{equation*}

    At this point, we have all the ingredients to proof the statement by contradiction. Assume that there is no pair $(a_1, a_2) \in \cA_f \times \cA_g$ such that the condition holds, then it must be that
    \begin{equation*}
    \label{eq:contradiction_equation}
        f(a_2) g(a_1) - \bar{F} \big( g(a_1) - g(a_2) \big) > \bar{G} \big( f(a_2) - f(a_1) \big) + f(a_1) g(a_2)
    \end{equation*}
    for every pair $(a_1, a_2) \in \cA_f \times \cA_g$. Therefore, we can integrate over all such pairs and the inequality should still hold, namely
    \begin{align}
        \int_{\cA_f} \int_{\cA_g} \bigg[ &f(a_2) g(a_1) - \bar{F} \big( g(a_1) - g(a_2) \big) \bigg] \mathrm{d} \bQ(a_1) \mathrm{d} \bQ(a_2) \nonumber \\ 
        &>  \int_{\cA_f} \int_{\cA_g} \bigg[ \bar{G} \big( f(a_2) - f(a_1) \big) + f(a_1) g(a_2) \bigg] \mathrm{d} \bQ(a_1) \mathrm{d} \bQ(a_2).
        \label{eq:contradiction}
    \end{align}
    To show that~\eqref{eq:contradiction} cannot happen, we need to introduce some notation. Let $F^- \coloneqq \int_{\cA_f} f(a) \mathrm{d} \bQ(a)$ and $F^+ \coloneqq \int_{\cA_g} f(a) \mathrm{d} \bQ(a)$ and note that $F^+ + F^- = \bar{F}$. Similarly, $G^- \coloneqq \int_{\cA_g} g(a) \mathrm{d} \bQ(a)$ and $F^+ \coloneqq \int_{\cA_f} g(a) \mathrm{d} \bQ(a)$ and $G^+ +G^- = \bar{G}$. Using this notation, we can use Fubini's theorem in~\eqref{eq:contradiction} and re-write it as
    \begin{equation*}
        F^+G^+ - (F^+ + F^-)(G^+ - G^-) > (G^+ + G^-) (F^+ - F^-) + F^- G^-,
    \end{equation*}
    which can be simplified to
    \begin{equation*}
        F^- G^- > F^+ G^+
    \end{equation*}
    and which is impossible by the definition of $F^-$, $F^+$, $G^+$ and $G^-$, completing the contradiction and therefore the proof.
\end{proof}

\begin{lemma}[{\cite[Lemma 5.13]{van_handel_probability_2016}}]
\label{lemma:van_handel}

Let $\overline{\mathbf{B}_d(0,1)}$ denote the $d$-dimensional closed Euclidean unit ball. We have $|\cN(\overline{\mathbf{B}_d(0,1)},||\cdot||_2,\varepsilon) =1 $ for $\varepsilon\geq 1$ and 
\begin{align*}
    &\bigg(\frac{1}{\varepsilon}\bigg)^d\leq|\cN(\overline{\mathbf{B}_d(0,1)},||\cdot||_2,\varepsilon)|\leq  \bigg(1+\frac{2}{\varepsilon}\bigg)^d &&\textnormal{for   } 0<\varepsilon<1.
\end{align*}

\end{lemma}

\section{Proofs}

\subsection{Proof of Proposition~\ref{prop:existence_learning}}
\label{app:existence_learning}

For each time step $t\in\{1,\ldots,T\}$, we will construct the sequence of function $\{f_t^k\}_{k=k_0}^{\infty}$ by induction and instead of constructing a sequence that satisfies directly \ref{prop:subprop_no_more_information}, we will design it such that for each $k > k_0$, it satisfies simultaneously the two following equations: 
\begin{align}\mi_t\big(A^\star_k;R(f_t^{k}(\hat{A}_{t,k})),R(f_t^{k-1}(\hat{A}_{t,k-1}))\Big) &\leq \mi_t\big(A^\star_k;R(\hat{A}_t),R(f_t^{k-1}(\hat{A}_{t,k-1}))\big)  \textnormal{ and}\label{eq:split_requirement_1} \\
\mi_t\big(A^\star_{k+1};R(\hat{A}_t),R(f_t^{k}(\hat{A}_{t,k}))\big) &\leq
\mi_t\big(A^\star_{k+1};R(\hat{A}_t),R(\hat{A}_t')\big),\label{eq:split_requirement_2}
\end{align}
thus ensuring that $f_t^k$ satisfies \ref{prop:subprop_no_more_information}.

First, we start by showing that there exists a function $f^{k_0}_t$ that satisfies requirement \ref{prop:subprop_singleton} and equation \eqref{eq:split_requirement_2}. By definition of $k_0$, we have that the cardinality of $\cA_{k_0}$ is $1$, that is $\cA_{k_0}= \{a_0\}$ for some $a_0 \in \cA$ and, as $A^\star_{k_0}\in \cA_{k_0}$ and $\hat{A}_{t,k_0}\in \cA_{k_0}$, we have $A^\star_{k_0}=\hat{A}_{t,k_0}=a_0$, thus satisfying requirement \ref{prop:subprop_singleton}. Setting the random function $f_t^{k_0}$ to have the same conditional probability distribution as $\bP_{A^\star|H^t}$ ensures equation \eqref{eq:split_requirement_2} is satisfied.

Now, we assume that for each $k \in \{ k_0,\ldots,K-1 \}$, we have constructed a function $f_{t}^k$ that satisfied \eqref{eq:split_requirement_1} and \eqref{eq:split_requirement_2}. We then want to show that we can construct a random function $f_t^{K}$ that also satisfies \eqref{eq:split_requirement_1} and \eqref{eq:split_requirement_2}.\\

First, for each  $a_{K,i}\in\cA_K$ with $i \in \{ 1,\ldots,|\cA_K|\}$, we define  $\cA_{K,i} = \{a\in\cA:\pi_K(a)=a_{K,i}\}$ as the set of actions in $\cA$ that are mapped to $a_{K,i}$ by the mapping $\pi_K$ associated to $\cA_K$, that is formally.
In this way, for each $a_{K,i}\in\cA_K$, we can write 
\begin{align*}             &\mi_t\big(A^\star_K;R(\hat{A}_t), R(f_t^{K-1}(\hat{A}_{t,K}))|\hat{A}_{t} \in \cA_{K,i}\big)\\
    &=\sum_{a \in \cA_{K,i}} \bP_t[\hat{A}_t =a |\hat{A}_{t} \in \cA_{K,i} ]\mi_t\big(A^\star_K;R(a), R(f_t^{K-1}(\hat{A}_{t,K}))|\hat{A}_{t} \in \cA_{K,i}\big)\\
    &=\sum_{a \in \cA_{K,i}} \bP_t[\hat{A}_t =a |\hat{A}_{t} \in \cA_{K,i} ]\mi_t\big(A^\star_K;R(a), R(f_t^{K-1}(\hat{A}_{t,K}))\big)
\end{align*}
and 
\begin{align*}
&\mi_t\big(A^\star_{K+1};R(\hat{A}_t), R(\hat{A}_t')|\hat{A}_{t} \in \cA_{K,i}\big)\\
    &=\sum_{a \in \cA_{K,i}} \bP_t[\hat{A}_t =a |\hat{A}_{t} \in \cA_{K,i} ]\mi_t\big(A^\star_{K+1};R(a), R(\hat{A}_t')|\hat{A}_{t} \in \cA_{K,i}\big)\\
    &=\sum_{a \in \cA_{K,i}} \bP_t[\hat{A}_t =a |\hat{A}_{t} \in \cA_{K,i} ]\mi_t\big(A^\star_{K+1};R(a), R(\hat{A}_t')\big),
\end{align*}
where we used the fact that $A^\star_K$ and $A^\star_{K+1}$ are independent of $\hat{A}_t$ when conditioned on $H^t$. \\

Applying Lemma~\ref{lemma:new_lemma}, for each step $t\in\{1,\ldots,T\}$ and each $a_{K,i}\in \cA_K$, there exist two actions $a_{K,i}^{t,1},a_{K,i}^{t,2} \in \cA_{K,i}$ and a value $p_{K,i}^t \in [0,1]$, such that: 
\begin{align*}
   \mi_t\big(A^\star_K;&R(\hat{A}_t), R(f_t^{K-1}(\hat{A}_{t,K}))|\hat{A}_{t} \in \cA_{K,i}\big)  \\ &\geq p_{K,i}^t \mi_t(A^\star_K;R(a_{K,i}^{t,1}),R(f_t^{K-1}(\hat{A}_{t,K}))\big)
    +(1-p_{K,i}^t) \mi_t(A^\star_K;R(a_{K,i}^{t,2}),R(f_t^{K-1}(\hat{A}_{t,K}))\big)
\end{align*}
and 
\begin{align*}
    \mi_t\big(A^\star_{K+1};&R(\hat{A}_t), R(\hat{A}_t')|\hat{A}_{t} \in \cA_{K,i}\big) & \\
    &\geq p_{K,i}^t \mi_t(A^\star_K;R(a_{K,i}^{t,1}),R(\hat{A}_t')\big)
    +(1-p_{K,i}^t) \mi_t(A^\star_K;R(a_{K,i}^{t,2}),R(\hat{A}_t')\big).
\end{align*}
For $a\in \cA_{K,i}$, we define the random function $f_t^K(a)$ such that it outputs  $a_{K,i}^{t,1}\in \cA_{K,i}$ with probability $p_{K,i}^t$ and $a_{K,i}^{t,2}\in \cA_{K,i}$ with probability $1-p_{K,i}^t$. We observe that for $a\in \cA_{K,i}$, $\pi_K(a) =  \pi_k(f_t^K(a)) = a_{K,i}$ as both $a$ and $f_t^K(a)$ belong to $\cA_{K,i}$. Then, the distance $\rho(a,f_t^k(a))$ is bounded by $2^{-K}$.   We repeat this procedure for all $a_{K,i}\in \cA_K$ and their corresponding $\cA_{K,i}$ to define $f_t^K(a)$ for all $a\in \cA$ and it holds by that, for all $a\in \cA$, $\rho(f_t^K(a),a)\leq 2^{-K}$. \\
We can verify that 
\begin{align*}
    &\mi_t\big(A^\star_K;R(f_t^{K}(\hat{A}_{t,K})),R(f_t^{K-1}(\hat{A}_{t,K-1}))\big)\\
    &=\sum_{a_{K,i}\in\cA_K} \sum_{j=1,2} \bP_t[f_t^{K}(\hat{A}_{t,K})) = a_{K,i}^{t,j}|\hat{A}_t\in \cA_{K,i}] \cdot \bP_t[\hat{A}_t\in \cA_{K,i}]\cdot\mi_t\big(A^\star_K;R(a_{K,i}^{t,j}),R(f_t^{K-1}(\hat{A}_{t,K-1}))\big)\\
    &= \sum_{a_{K,i}\in\cA_K} \bP_t[\hat{A}_t\in \cA_{K,i}] (p_{K,i}^t \cdot \mi_t\big(A^\star_K;R(a_{K,i}^{t,1}),R(f_t^{K-1}(\hat{A}_{t,K-1}))\big)\\
    &\qquad+(1-p_{K,i}^t) \cdot \mi_t\big(A^\star_K;R(a_{K,i}^{t,2}),R(f_t^{K-1}(\hat{A}_{t,K-1}))\big)\\
    &\leq \sum_{a_{K,i}\in\cA_K} \bP_t[\hat{A}_t\in \cA_{K,i}] \mi_t\big(A^\star_K;R(\hat{A}_t),R(f_t^{K-1}(\hat{A}_{t,K-1}))|\hat{A}_t\in \cA_{K,i}\big)\\
    &= \mi_t\big(A^\star_K;R(\hat{A}_t),R(f_t^{K-1}(\hat{A}_{t,K-1}))\big)
\end{align*}
and similarly that 
\begin{align*}
    &\mi_t\big(A^\star_{K+1};R(f_t^{K}(\hat{A}_{t,K})),R(\hat{A}_{t}')\big)\\
    &=\sum_{a_{K,i}\in\cA_K} \sum_{j=1,2} \bP_t[f_t^{K}(\hat{A}_{t,K})) = a_{K,i}^{t,j}|\hat{A}_t\in \cA_{K,i}] \cdot \bP_t[\hat{A}_t\in \cA_{K,i}]\cdot\mi_t\big(A^\star_{K+1};R(a_{K,i}^{t,j}),R(\hat{A}_t')\big)\\
    &= \sum_{a_{K,i}\in\cA_K} \bP_t[\hat{A}_t\in \cA_{K,i}] (p_{K,i}^t \cdot \mi_t\big(A^\star_{K+1}R(a_{K,i}^{t,1}),R(\hat{A}_t')\big) +(1-p_{K,i}^t) \cdot \mi_t\big(A^\star_{K+1};R(a_{K,i}^{t,2}),R(\hat{A}_t')\big)\\
    &\leq \sum_{a_{K,i}\in\cA_K} \bP_t[\hat{A}_t\in \cA_{K,i}] \mi_t\big(A^\star_{K+1};R(\hat{A}_t),R(\hat{A}_t')|\hat{A}_t\in \cA_{K,i}\big)\\
    &= \mi_t\big(A^\star_{K+1};R(\hat{A}_t),R(\hat{A}_t')\big)
\end{align*}
where the inequalities follow from the construction of $f_t^K$. Thus $f_t^K$ satisfies requirement \ref{prop:subprop_no_more_information}. As the result holds already for $k=k_0$, by induction, we extend this result for all $k\geq k_0$.\\

We note that by construction, for each step $t\in\{1,\ldots,T\}$ and for each $k\geq k_0$, we have that
\begin{align}
\rho(f_t^k(A^\star_k),A^\star)\leq \rho(f_t^k(A^\star_k),A^\star_k)+\rho(A^\star_k,A^\star)\leq2\cdot 2^{-k}, \label{eq:proof_existence_equation_1}\\
\rho(f_t^k(\hat{A}_{t,k}),\hat{A}_t)\leq \rho(f_t^k(\hat{A}_{t,k}),\hat{A}_{t,k})+\rho(\hat{A}_{t,k},\hat{A}_t)\leq2\cdot 2^{-k},\label{eq:proof_existence_equation_2}
\end{align}
where we use the triangle inequality together with the definition of $f_t^k$ and of $A^\star_k$ and $\hat{A}_{t,k}$.\\

Lastly, we have to verify that at each period $t\in\{1,\ldots,T\}$, the regret of the ``\emph{approximate learning}" asymptotically converges to the regret of Two Steps Thompson Sampling regret for finer approximations.\\

Using the fact that by construction of $f_t^{k}$, we have for all $a\in \cA_{k}$ that $\pi_{k}(f_t^{k}(a))=a$ and that by definition $A^\star_{k} = \pi_{k}(A^\star)$, we can write: 
\begin{align*}
        \bE_t [R(f_t^{k}(A^\star_{k}))-R(A^\star)] &= \bE_t [R(f_t^{k}(A^\star_{k}))-R(A^\star_{k})] + \bE_t [R(A^\star_{k})-R(A^\star)]\\
        &=  \bE_t [R(f_t^{k}(A^\star_{k}))-R(\pi_{k}(f_t^{k}(A^\star_{k})))] + \bE_t [R(\pi_{k}(A^\star))-R(A^\star)]\\
        &\leq 2\cdot \bE_t[\sup_{a\in \cA} R(\pi_k(a))-R(a)].
\end{align*}
Since the process is separable, by using the same argument as in the proof of~\cite[Theorem 5.24]{van_handel_probability_2016}, we have that 
\begin{align*}
    \lim_{k\to\infty} \bE_t[\sup_{a\in \cA} R(\pi_k(a))-R(a)] =0, 
\end{align*}
and therefore
\begin{align*}
    \lim_{k\to\infty} \bE_t [R(f_t^{k}(A^\star_{k}))] = \bE_t[R(A^\star)].
\end{align*}

A similar analysis can be applied to $ \bE_t [R(f_t^{k}(\hat{A}_{t,k}))-R(\hat{A}_{t})]$ and leads to 
\begin{align*}
    \lim_{k\to \infty} \bE_t [R(f_t^{k}(A^\star_{k})) - R(f_t^{k}(\hat{A}_{t,k}))] = \bE_t [R(A^\star) - R(\hat{A}_t)].
\end{align*}

\subsection{Proof of Proposition~\ref{prop:chain-link_information_ratio_bound_for_smooth_d_linear_bandits}}
\label{app:chain-link_information_ratio_bound_for_smooth_d_linear_bandits}

We start the proof by recalling the definition of $\Gamma_{t,k}$ as
\begin{equation*}
\Gamma_{t,k} = \frac{\bE_t \Big[\Big(R(f_t^k(A^\star_k))- R(f_t^k(\hat{A}_{t,k}))  \Big)-\Big(R(f_t^{k-1}(A^\star_{k-1}))-R(f_t^{k-1}(\hat{A}_{t,k-1})) \Big)\Big]^2}{\mi_t(f_t^k(A^\star_k),f_t^{k-1}(A^\star_{k-1});R(f_t^k(\hat{A}_{t,k})),R(f_t^{k-1}(\hat{A}_{t,k-1})))}
\end{equation*}
where $A^\star_k$ and $\hat{A}_{t,k}$ are the $k^{\textnormal{th}}$-quantizations respectively of the optimal action $A^\star$ and the sampled action $\hat{A}_t$. We recall from the proof of Proposition~\ref{prop:existence_learning} that the definition of $f_t^k(A)$ implies that for all $a_{k,m} \in \cA_k$ there exist a pair of actions $a_{k,m}^{t,1},a_{k,m}^{t,2} \in \cA_{k,m}$ such that 
\begin{align*}
    \bP_t[f_t^k(A) = a_{k,m}^{t,1}|A\in \cA_{k,m} ] = p^t_{k,m}, \quad  \bP_t[f_t^k(A) = a_{k,m}^{t,2}|A\in \cA_{k,m} ] = 1-p^t_{k,m}.
\end{align*}

For the sake of brevity, we define the notation 
\begin{align*}
    \bQ_t[a_{k-1,m},a_{k,l},i,i']\coloneqq 
    &\bP_t[f_t^{k-1}(A^\star_{k-1}) = a_{k-1,m}^{t,i}|A^\star_{k-1} \in \cA_{k,m}]  \\
    &\cdot \bP_t[f_t^{k}(A^\star_{k}) = a_{k,l}^{t,i'}|A^\star_{k} \in \cA_{k,l}]  \\
    &\cdot \bP_t[A^\star_{k} \in \cA_{k,l},A^\star_{k-1} \in \cA_{k-1,m}]
\end{align*}
and use the notation $\{(a_{k-1,\delta_n},a_{k,\gamma_n},i_{\mu_n},i_{\nu_n}')\}_{n=1}^{N_k}$ to represent the sequence of all quadruplets $\{a_{k-1},a_{k},i,i'\}$ such that $a_{k-1}\in \cA_{k-1},a_{k} \in \cA_{k}, i\in \{1,2\},i'\in\{1,2\} \textnormal{  and } \pi_{k-1}(a_{k})=a_{k-1}$, where $N_k$ is the number of such quadruplets.\\ 

We will first focus on $$\bE_t \left[ \big( R( f_t^k(A^\star_k) ) - R( f_t^{k-1}(A^\star_{k-1}) ) \big) - \big( R ( f_t^k(\hat{A}_{t,k}) ) - R ( f_t^{k-1}(\hat{A}_{t,k-1}) ) \big) \right]$$ and note that we can relate it to the trace of a random matrix. Indeed, using the previously introduced notations, we can write this expectation as
\begin{align*}
    &\sum_{n=1}^{N_k} \bQ_t[a_{k-1,\delta_n},a_{k,\gamma_n},i_{\mu_n},i_{\nu_n}'] \\
    &\cdot \big(\bE_t [ R( a^{t,i_{\mu_n}}_{k,\gamma_n} )  - R( a^{t,i_{\nu_n}'}_{k-1,\delta_n} )|f_t^{k}(A^\star_{k}) = a^{t,i_{\mu_n}}_{k,\gamma_n},f_t^{k-1}(A^\star_{k-1}) = a^{t,i_{\nu_n}'}_{k-1,\delta_n}  ] - \bE_t [ R( a^{t,i_{\mu_n}}_{k,\gamma_n} )  - R( a^{t,i_{\nu_n}'}_{k-1,\delta_n} )] \big).
\end{align*}

Therefore, for any round $t \in \{1,\ldots,T\}$, conditioned on the history $\hat{H}^t$, we can define a random matrix $M^{k,t} \in \bR^{N_k \times N_k}$ by specifying the entry $M^{k,t}_{p,q}$ to be equal to
\begin{align*}
    &\sqrt{\bQ_t[a_{k-1,\delta_p},a_{k,\gamma_p},i_{\mu_p},i_{\nu_p}']} \sqrt{\bQ_t[a_{k-1,\delta_q},a_{k,\gamma_q},i_{\mu_q},i_{\nu_q}']}\\
    &\big( \bE_t [ R( a^{t,i_{\mu_q}}_{k,\gamma_q} )  - R( a^{t,i_{\nu_q}'}_{k-1,\delta_q} )|f_t^{k}(A^\star_{k}) = a^{t,i_{\mu_p}}_{k,\gamma_p},f_t^{k-1}(A^\star_{k-1}) = a^{t,i_{\nu_p}'}_{k-1,\delta_p} ] - \bE_t [ R( a^{t,i_{\mu_q}}_{k,\gamma_q} )  - R( a^{t,i_{\nu_q}'}_{k-1,\delta_q} )] \big)
\end{align*}
for all $p,q = 1, \ldots, N_k$. In this way, the trace of the matrix $M^{k,t}$ is equal to the desired expectation, namely $$\mathrm{Tr}(M^{k,t}) = \bE_t \left[ ( R( f_t^k(A^\star_k) ) - R( f_t^{k-1}(A^\star_{k-1}) ) ) - ( R ( f_t^k(\hat{A}_{t,k}) ) - R ( f_t^{k-1}(\hat{A}_{t,k-1}) ) ) \right].$$

Here, we can note that $R( f_t^k(A^\star_k))  - R( f_t^{k-1}(A^\star_{k-1}) )$ is $(6\cdot2^{-k})^2$-sub-Gaussian. Indeed, by construction, of $f_t^k(A^\star_k)$ and $f_t^{k-1}(A^\star_{k-1})$, we had showed in \eqref{eq:proof_existence_equation_1} that $\rho(f_t^k(A^\star_k),A^\star) \leq 2\cdot 2^{-k}$ and $\rho(f_t^{k-1}(A^\star_{k-1}),A^\star) \leq 2\cdot 2^{-(k-1)}$. Then, by using the triangle inequality, we have that
\begin{align*}
    \rho(f_t^k(A^\star_k),f_t^{k-1}(A^\star_{k-1}))\leq \rho(f_t^k(A^\star_k),A^\star)+\rho(A^\star,f_t^{k-1}(A^\star_{k-1})) \leq 2\cdot 2^{-k} + 2\cdot 2^{-(k-1)} = 6\cdot 2^{-k}.
\end{align*}
Similarly, we can show that $R( f_t^k(\hat{A}_{t,k}))  - R( f_t^{k-1}(\hat{A}_{t,k-1}) )$ is also $(6\cdot2^{-k})^2$-sub-Gaussian.\\

In the same fashion as in~\cite[Proposition 5]{russo_information-theoretic_2015}, we relate the mutual information $$\mi_t(f_t^k(A^\star_k), f_t^{k-1}(A^\star_{k-1}); R(f_t^k(\hat{A}_{t,k})), R(f_t^{k-1}(\hat{A}_{t,k-1})))$$ to the squared Frobenius norm of $M^{k,t}$ as: 
\begin{align*}
    &\mi_t(f_t^k(A^\star_k), f_t^{k-1}(A^\star_{k-1}); R(f_t^k(\hat{A}_{t,k})), R(f_t^{k-1}(\hat{A}_{t,k-1})))\\
    &\geq\mi_t(f_t^k(A^\star_k), f_t^{k-1}(A^\star_{k-1}); R(f_t^k(\hat{A}_{t,k}))- R(f_t^{k-1}(\hat{A}_{t,k-1})))\\
    &=\sum_{p=1}^{N^k}\sum_{q=1}^{N^k} \bQ_t[a_{k-1,\delta_p},a_{k,\gamma_p},i_{\mu_p},i_{\nu_p}']\bQ_t[a_{k-1,\delta_q},a_{k,\gamma_q},i_{\mu_q},i_{\nu_q}'] \\
    &\cdot \kl(\bP_{R( a^{t,i_{\mu_q}}_{k,\gamma_q} )  - R( a^{t,i_{\nu_q}'}_{k-1,\delta_q} )|\hat{H}^t,f_t^k(A^\star_k) = a^{t,i_{\mu_p}}_{k,\gamma_p}, f_t^{k-1}(A^\star_{k-1}) = a^{t,i_{\nu_p}'}_{k-1,\delta_p} }||\bP_{R( a^{t,i_{\mu_q}}_{k,\gamma_q} )  - R( a^{t,i_{\nu_q}'}_{k-1,\delta_q} )|\hat{H}^t} ) \\
    & \geq \sum_{p=1}^{N^k}\sum_{q=1}^{N^k} \bQ_t[a_{k-1,\delta_p},a_{k,\gamma_p},i_{\mu_p},i_{\nu_p}']\bQ_t[a_{k-1,\delta_q},a_{k,\gamma_q},i_{\mu_q},i_{\nu_q}']\cdot \frac{1}{2\cdot(6\cdot2^{-k})^2}   \\
    &\cdot\big( \bE_t [ R( a^{t,i_{\mu_q}}_{k,\gamma_q} )  - R( a^{t,i_{\nu_q}'}_{k-1,\delta_q} )|f_t^k(A^\star_k) = a^{t,i_{\mu_p}}_{k,\gamma_p}, f_t^{k-1}(A^\star_{k-1}) = a^{t,i_{\nu_p}'}_{k-1,\delta_p} ] - \bE_t [ R( a^{t,i_{\mu_q}}_{k,\gamma_q} )  - R( a^{t,i_{\nu_q}'}_{k-1,\delta_q} )] \big)^2\\
    & = \frac{1}{2(6\cdot2^{-k})^2} ||M^{k,t}||^2_F 
\end{align*}
where the last inequality is obtained again using the Donsker--Varadhan inequality~\cite[Theorem~5.2.1]{gray_entropy_2013} as in~\cite[Lemma 3]{russo_information-theoretic_2015}.

Combining the last two equations and using the inequality $\textnormal{trace}(M) \leq \sqrt{\textnormal{rank}(M)} ||M||_F$ ~\cite[Fact~10]{russo_information-theoretic_2015}, it comes that
\begin{align*}
\Gamma_{t,k}\leq 2 (6\cdot2^{-k})^2 \frac{\textnormal{Trace}(M^{k,t})^2}{||M^{k,t}||_{F}^2}\leq 2 (6\cdot2^{-k})^2 \cdot \textnormal{rank}(M^{k,t}) \textnormal{ a.s.}.
\end{align*}
We conclude the proof by showing that the rank of the matrix $M^{k,t}$ is upper bounded by $d$.

For the sake of brevity, we define $\Theta_t \coloneqq \bE_t[\Theta]$ and for $n=1,\ldots,N_k$, we define $\bQ_{n,t} = \bQ_t[a_{k-1,\delta_n},a_{k,\gamma_n},i_{\mu_n},i_{\nu_n}']$ and $\Theta_{n,t} = \bE_t[\Theta|f_t^k(A^\star_k) = a^{t,i_{\mu_n}}_{k,\gamma_n}, f_t^{k-1}(A^\star_{k-1}) = a^{t,i_{\nu_n}'}_{k-1,\delta_n}]$. \\

We then have 
\begin{align*}
    \bE_t \left[ R( a^{t,i_{\mu_q}}_{k,\gamma_q} )  - R( a^{t,i_{\nu_q}'}_{k-1,\delta_q} ) \right] = \bE_t \left[ \innerproduct{a^{t,i_{\mu_q}}_{k,\gamma_q} }{\Theta}-\innerproduct{a^{t,i_{\nu_q}'}_{k-1,\delta_q}}{\Theta} \right] = \innerproduct{a^{t,i_{\mu_q}}_{k,\gamma_q} -a^{t,i_{\nu_q}'}_{k-1,\delta_q}}{\Theta_t} 
\end{align*}
and 
\begin{align*}
    &\bE_t \left[ R( a^{t,i_{\mu_q}}_{k,\gamma_q} )  - R( a^{t,i_{\nu_q}'}_{k-1,\delta_q} ) |f_t^k(A^\star_k) = a^{t,i_{\mu_p}}_{k,\gamma_p}, f_t^{k-1}(A^\star_{k-1}) = a^{t,i_{\nu_p}'}_{k-1,\delta_p} \right] \\
    &= \bE_t \left[ \innerproduct{a^{t,i_{\mu_q}}_{k,\gamma_q} }{\Theta}-\innerproduct{a^{t,i_{\nu_q}'}_{k-1,\delta_q}}{\Theta}|f_t^k(A^\star_k) = a^{t,i_{\mu_p}}_{k,\gamma_p}, f_t^{k-1}(A^\star_{k-1}) = a^{t,i_{\nu_p}'}_{k-1,\delta_p} \right] \\
    &= \innerproduct{a^{t,i_{\mu_q}}_{k,\gamma_q} -a^{t,i_{\nu_q}'}_{k-1,\delta_q}}{\Theta_{p,t}} 
\end{align*}
Since the inner product is linear, we can rewrite each entry $M^{k,t}_{p,q}$ of the matrix $M^{k,t}$ as
\begin{align*}
    \sqrt{\bQ_{p,t} \bQ_{q,t}}\innerproduct{a^{t,i_{\mu_q}}_{k,\gamma_q} -a^{t,i_{\nu_q}'}_{k-1,\delta_q}}{\Theta_{p,t}-\Theta_t} .
\end{align*}
Equivalently, the matrix $M^{k,t}$ can be written as
\begin{align*}
    \begin{bmatrix}
    \sqrt{\bQ_{1,t}}(\Theta_{1,t}-\Theta_t)\\
    \vdots\\
    \sqrt{\bQ_{N_k,t}}(\Theta_{N_k,t}-\Theta_t)
    \end{bmatrix}
    \begin{bmatrix}
    \sqrt{\bQ_{1,t}}\big(a^{t,i_{\mu_1}}_{k,\gamma_1} - a^{t,i_{\nu_1}'}_{k-1,\delta_1}\big)&
    \cdots &
    \sqrt{\bQ_{N_k,t}}\big(a^{t,i_{\mu_{N_k}}}_{k,\gamma_{N_k}} - a^{t,i_{\nu_{N_k}}'}_{k-1,\delta_{N_k}}\big)
    \end{bmatrix}.
\end{align*}
This rewriting highlights that $M^{k,t}$ can be written as the product of a $N_k$ by $d$ matrix and a $d$ by $N_k$ matrix and therefore has a rank lower or equal than $\min(d,N_k)$.\\

For completeness, we can write that the chain-link information ratio is upper bounded by $\Gamma_{t,k}\leq 2\cdot\rho_k^2\cdot d$ where $\rho_k$ is an upper bound on $\rho(f_t^k(A^\star_k),f_t^{k-1}(A^\star_{k-1}))$. This remark will be of use in the proof of Proposition \ref{prop:linear_smooth_bandit_unit_ball}.

\subsection{Proof of Corollary~\ref{cor:entropy_integral}}
\label{app:entropy_integral}

Bounding the entropy of $A^\star_k$ by the cardinality of set $\cA_k$, we have that
\begin{align*}
    \sum_{k = k_0 + 1}^\infty 2^{-k}\sqrt{\ent(A^\star_k)} \leq \sum_{k = k_0 + 1}^\infty 2^{-k} \sqrt{\log(|\cN(\cA,\rho, 2^{-k})|)}.
\end{align*}

By definition of the $\varepsilon$-net,  $|\cN(\cA,\rho, \varepsilon)|$ is decreasing in $\varepsilon$. It then comes that
\begin{align*}
    \sum_{k = k_0 + 1}^\infty 2^{-k} \sqrt{\log(|\cN(\cA,\rho, 2^{-k})|)} &= 2 \sum_{k = k_0 + 1}^\infty \int_{2^{-k}}^{2^{-k-1}} \sqrt{\log(|\cN(\cA,\rho, 2^{-k})|)} \, d\varepsilon \\
&\leq 2 \sum_{k = k_0 + 1}^\infty \int_{2^{-k}}^{2^{-k-1}} \sqrt{\log(|\cN(\cA,\rho, \varepsilon)|)} \, d\varepsilon \\
&= 2 \int_{0}^{\mathrm{diam}(\cA)} \sqrt{\log(|\cN(\cA,\rho, \varepsilon)|)} \, d\varepsilon.\\
&= 2 \int_{0}^{\infty} \sqrt{\log(|\cN(\cA,\rho, \varepsilon)|)} \, d\varepsilon,
\end{align*}
where the last equality comes from the fact that $\cN(\cA, \rho, \varepsilon)$ is a singleton for every $\varepsilon > \mathrm{diam}(\cA)$.

Using this fact together with \Cref{thm:main_theorem} yields the desired result.

\subsection{Proof of Proposition~\ref{prop:linear_smooth_bandit_unit_ball}}
\label{app:linear_smooth_bandit_unit_ball}

In the end of proof of Proposition \ref{prop:chain-link_information_ratio_bound_for_smooth_d_linear_bandits}, we have showed that the chain-link information ratio was in general bounded by $\Gamma_{t,k}\leq 2\cdot\rho_k^2\cdot d$ where $\rho_k$ is an upper bound on $\rho(f_t^k(A^\star_k),f_t^{k-1}(A^\star_{k-1}))$ and proved that by definition of the quantizations $A^\star_k$ and the sampling functions $f_t^k$, it holds that 
    $$\rho(f_t^k(A^\star_k),f_t^{k-1}(A^\star_{k-1})) \leq 2\cdot 2^{-k} + 2\cdot 2^{-(k-1)}.$$

    We can reflect that the choice of using $2^{-k}$-nets to define our sequence of quantizations $\{A^\star_k\}_{k=k_0+1}^{\infty}$ was arbitrary. In general, we could have considered a $\alpha^{-k}$-net for some $\alpha>1$. Adapting the bound on $\rho_k$ and to that reflection, leads to the following bound: 
    $$\rho(f_t^k(A^\star_k),f_t^{k-1}(A^\star_{k-1})) \leq 2\cdot \alpha^{-k} + 2\cdot \alpha^{-(k-1)}.$$\\
    Combining this result with \Cref{thm:main_theorem}, we get that
    \begin{align*}
        \REG^{\textnormal{2-TS}}_T &\leq 2 \sum_{k=k_0+1}^\infty \sqrt{2\cdot\rho_k^2\cdot d \cdot T  \cdot \log(|\cN(\cA,\rho,\alpha^{-k})|) },
    \end{align*}
    where we upper bounded the entropy of $A^\star_k$ by the logarithm of the  cardinality of the set $\cA_k$.\\
    
    Applying Lemma~\ref{lemma:van_handel} to upper bound the cardinality of the smallest $\alpha^{-k}$-net $\cN(\cA,\rho,\alpha^{-k})$ and rearanging the terms, we get the following bound: 
    \begin{align*}
       \REG^{\textnormal{2-TS}}_T &\leq 2\cdot d\cdot \sqrt{T} \sum_{k=k_0+1}^\infty \sqrt{2\cdot\rho_k^2  \cdot \log(2\cdot \alpha^k + 1) }.
    \end{align*}

    Now, we note that for linear bandit problems, we can define the first quantization set $\cA_{k_0}$ to be the center of the ball, that is $\cA_{k_0}=\{0_d\}$ where $0_d$ is the $d$-dimensional zero and chose $f_t^{k_0}(0_d)= 0_d$. It is easy to verify that this choice satisfies Proposition \ref{prop:existence_learning} \ref{prop:subprop_singleton} as $A^\star_{k_0} = \hat{A}_{t,k_0}= 0_d$ and $f_t^{k_0}(A^\star_{k_0})=f_t^{k_0}(\hat{A}_{t,k_0})=0_d$, as well as fulfills \ref{eq:split_requirement_2} as $R(f_t^{k_0}(\hat{A}_{t,k_0}))=R(0_d)$ does not depend on the $\Theta$ and therefore is independent of $A^\star$ and $A^\star_{k_0+1}$.\\
    
    Observing that in the unit ball, by definition the radius is $1$, we first note that $\cA_{k_0}$ is a $(\alpha^0)$-net for $\cA$, implying $k_0=0$ and secondly that $\rho(f_t^{k_0+1}(A^\star_{k_0+1}),f_t^{k_0}(A^\star_{k_0})) = \rho(f_t^{k_0+1}(A^\star_{k_0+1}),0_d)\leq 1$ and therefore we can use $\rho_{k_0+1}=1$ which is a better upper bound than $2\cdot \alpha^{-(k_0+1)} + 2\cdot \alpha^{-k_0} = 2\cdot(1+ \alpha^{-1})$.  \\

    Applying those results, we obtain the following bound: 
     \begin{align*}
       \REG^{\textnormal{2-TS}}_T &\leq d \sqrt{T} \cdot2\cdot\bigg(\sqrt{2\cdot \log(2 \alpha + 1) }+\sum_{k=2}^\infty (2\cdot \alpha^{-k} + 2\cdot \alpha^{-(k-1)})\sqrt{2\cdot\log(2 \alpha^k + 1) }\bigg).
    \end{align*}

    For instance, choosing $\alpha = 20$, we have that
    $$2\cdot\bigg(\sqrt{2\cdot \log(2 \alpha + 1) }+\sum_{k=2}^\infty (2\cdot \alpha^{-k} + 2\cdot \alpha^{-(k-1)})\sqrt{2\cdot\log(2 \alpha^k + 1) }\bigg) \approx 6.27.$$ 
    Finally, rounding up this value leads to claimed result.
\end{document}